\RequirePackage{fix-cm}
\documentclass[smallextended]{svjour3}       
\smartqed  
\usepackage{graphicx}
\usepackage{breakcites}
\usepackage[subrefformat=parens]{subcaption}
\usepackage[subrefformat=parens]{subcaption}
\usepackage[ruled, vlined, linesnumbered]{algorithm2e}
\usepackage{lmodern}
\usepackage{amsmath}
\usepackage{amsfonts}
\usepackage{mathtools}
\usepackage{amsthm}
\SetKwProg{Init}{init}{}{}
\usepackage{scalerel,stackengine}
\newtheorem{lem}{Lemma}

\theoremstyle{definition} \newtheorem{defn}{Definition}
\newcommand{\indepe}{\mathop{\perp\!\!\!\perp}}
\newcommand{\notindepe}{\mathop{\perp\!\!\!\!\!\!/\!\!\!\!\!\!\perp}}
\stackMath
\newcommand\reallywidehat[1]{%
\savestack{\tmpbox}{\stretchto{%
  \scaleto{%
    \scalerel*[\widthof{\ensuremath{#1}}]{\kern-.6pt\bigwedge\kern-.6pt}%
    {\rule[-\textheight/2]{1ex}{\textheight}}
  }{\textheight}%
}{0.5ex}}%
\stackon[1pt]{#1}{\tmpbox}%
}
\def\rightarrowCirc{\hbox{$\circ$}\kern-1.5pt\hbox{$\rightarrow$}}
\def\circHyphenCirc{\hbox{$\circ$}\kern-1.5pt\hbox{$-$}\kern-1.5pt\hbox{$\circ$}}
\def\circHyphen{\hbox{$\circ$}\kern-1.5pt\hbox{$-$}}

\theoremstyle{definition} \newtheorem{assumption}{Assumption}

\def\rightarrowCirc{\hbox{$\circ$}\kern-1.5pt\hbox{$\rightarrow$}}
\def\circHyphenCirc{\hbox{$\circ$}\kern-1.5pt\hbox{$-$}\kern-1.5pt\hbox{$\circ$}}
\def\circHyphen{\hbox{$\circ$}\kern-1.5pt\hbox{$-$}}

\begin{document}

\title{Use of Prior Knowledge to Discover Causal Additive Models with Unobserved Variables and its Application to Time Series Data
}


\author{Takashi Nicholas Maeda, Shohei Shimizu
}


\institute{Tokyo Denki University \email{tn.maeda@mail.dendai.ac.jp}           
}

\date{Received: date / Accepted: date}

\maketitle

\begin{abstract}
This paper proposes two methods for causal additive models with unobserved variables (CAM-UV). CAM-UV assumes that the causal functions take the form of generalized additive models and that latent confounders are present. First, we propose a method that leverages prior knowledge for efficient causal discovery. Then, we propose an extension of this method for inferring causality in time series data. The original CAM-UV algorithm differs from other existing causal function models in that it does not seek the causal order between observed variables, but rather aims to identify the causes for each observed variable. Therefore, the first proposed method in this paper utilizes prior knowledge, such as understanding that certain variables cannot be causes of specific others. Moreover, by incorporating the prior knowledge that causes precedes their effects in time, we extend the first algorithm to the second method for causal discovery in time series data. We validate the first proposed method by using simulated data to demonstrate that the accuracy of causal discovery increases as more prior knowledge is accumulated. Additionally, we test the second proposed method by comparing it with existing time series causal discovery methods, using both simulated data and real-world data.

\end{abstract}

\section{Introduction}
\label{intro}
Causal discovery refers to a special class of statistical and machine learning methods that infer causal relationships. These studies propose inferential methods deductively derived from assumptions about the data generation process, and the methods enable us to create causal graphs between observed variables without additional experiments. The assumptions of existing causal methods include acyclicity of causal graphs, absence of latent confounders, and independence and identical distribution of exogenous variables~\cite{Spirtes91,shimizu2006,shimizu2011,peters2014,NEURIPS2018_e347c514}. The methods have been applied to various types of data including economic data~\cite{LAI2015173}, meteorological data~\cite{ebert2012}, fMRI data~\cite{SMITH2011875}.

This paper proposes a causal discovery method for time-series data assuming the presence of latent confounders. Most existing methods for time-series data assume the absence of a latent confounder~\cite{JMLR:v9:chu08a,JMLR:v11:hyvarinen10a}. However, most data do not satisfy such assumption. A causal discovery method for time-series data, latent Peter-Clark momentary conditional independence (LPCMCI)~\cite{NEURIPS2020_94e70705}, assumes the presence of latent confounders. However, since LPCMCI is a constraint-based method, it cannot distinguish causal structures that entail the same set of conditional independence between variables. This paper aims to propose a causal functional model-based method for time-series data assuming the presence of latent confounders. We extend the causal additive models with unobserved variables (CAM-UV) algorithm~\cite{pmlr-v161-maeda21a, maeda2021discovery} to propose time-series CAM-UV (TS-CAM-UV), a method for causal discovery from time-series data with latent confounders. The original CAM-UV algorithm assumes that: (1) data are independently and identically distributed, (2) causal functions take the form of a generalized additive model of nonlinear functions, and (3) latent confounders are present. TS-CAM-UV, being a causal function model-based method, can identify causal relationships, provided the data fulfills its assumptions.

Causal discovery methods for time-series data represent the state of variable $X_i$ at time point $t$ as $X_i^t$ treating the states of $X_i$ at different points such as $X_i^t, X_i^{t-1}, \cdots, X_i^{s}$ as separate variables. This allows for representing causal relationships between variables at different time points.

Time series causal discovery methods can be described as causal discovery methods that utilize the prior knowledge that effects do not precede their causes in time. Therefore, before proposing the TS-CAM-UV algorithm, this paper proposes a method called CAM-UV with prior knowledge (CAM-UV-PK), which applies prior knowledge to CAM-UV. TS-CAM-UV is proposed as a method that introduces the knowledge that variables representing future states cannot be the cause of variables representing past states. To the best of our knowledge, this is the first method for time series causal discovery that adopts a causal function model approach assuming the presence of latent confounders. 

The contributions of this paper are as follows:
\begin{itemize}
	\item This paper proposes a method called the CAM-UV-PK algorithm, which can introduce prior knowledge in the form of statements such as $X_i$ cannot be a cause of $X_j$. The performance of the CAM-UV-PK algorithm is verified using simulation data.
	\item We propose a time-series causal discovery method called the TS-CAM-UV algorithm, which applies the prior knowledge that variables representing future states cannot be causes of variables representing past states. The performance of the TS-CAM-UV algorithm is verified using both simulation data and real-world data.
\end{itemize}

The remainder of this paper comprises the following. Section~2 reviews previous studies on causal discovery methods for i.i.d. data and time-series data. Section~3 introduces the models of the data generation processes of CAM-UV and TS-CAM-UV, followed by Section~4 which shows the identifiability of those models. Section~5 introduces the two proposed methods, the CAM-UV-PK algorithm and the TS-CAM-UV algorithm. Section~6 shows and discusses the results of the experiments of the proposed methods. Section~7 brings the paper to a conclusion.

\section{Related studies}
\label{related}
Causal discovery methods often assume that the causal structures form directed acyclic graphs (DAGs), that there is no latent confounders, and that data are independently and identically distributed~\cite{chickering2002,peters2014,shimizu2006,shimizu2011,Spirtes91}. The constrained-based methods including the Peter-Clark (PC) algorithm~\cite{Spirtes91} and the fast causal inference (FCI) algorithm~\cite{fci} infer causal relationships on the basis of conditional independence in the joint distribution. FCI identifies the presence of latent confounders whereas PC assumes the absence of unobserved common causes. PC and FCI cannot distinguish between the two causal graphs that entail exactly the same sets of conditional independence. Compared to constrained-based methods, causal functional model-based methods can identify the entire causal models under proper assumptions. Linear non-Gaussian acyclic models (LiNGAM)~\cite{shimizu2006,shimizu2011} assume that causal relationships are linear and the external effects are non-Gaussian. Additive noise models (ANMs) and causal additive models~\cite{peters2014} assume the causal relationships are nonlinear. Both LiNGAM and ANMs assume the absence of unobserved variables. Causal additive models with unobserved variables (CAM-UV)~\cite{pmlr-v161-maeda21a} are extended models of causal additive models (CAMs)~\cite{buhlmann2014} and assume that the causal functions take the form of generalized additive models (GAMs)~\cite{hastie1990generalized} and that unobserved variables are present.\par
Time-series causal discovery methods have been proposed as extensions of the above methods. The time-series FCI (tsFCI) algorithm~\cite{entner2010causal} and a structural vector autoregression FCI (SVAR-FCI)~\cite{pmlr-v92-malinsky18a} adapt FCI algorithm and use time order and stationarity to infer causal relationships. VAR-LiNGAM~\cite{JMLR:v11:hyvarinen10a} is based on LiNGAM and assumes the linearity of causal relationships, non-Gaussianity of external effects, and the absence of unobserved common causes. Time series models with independent noise (TiMINo)~\cite{NIPS2013_47d1e990} adapts ANMs, and it assumes the absence of latent confounders. The Peter-Clark momentary conditional independence (PCMCI) algorithm~\cite{doi:10.1126/sciadv.aau4996} is an adaptation of the conditional independence-based PC algorithm that addresses strong autocorrelations in time series via the use of a momentary conditional independence (MCI) test. Latent PCMCI (LPCMCI)~\cite{NEURIPS2020_94e70705} is an extension of PCMCI to include unobserved variables. However, to the best of our knowledge, no causal functional model-based method has been proposed for time-series data under the assumption that causal relationships are nonlinear and latent confounders are present.

\section{Models}
\subsection{CAM-UV: Causal additive models with unobserved variables}\label{section:modelCAM-UV}

Causal additive noise models with unobserved variables (CAM-UV) \cite{pmlr-v161-maeda21a, maeda2021discovery} are defined as the equation below:
\begin{equation}
	\label{eq:4}
    V_i=\sum_{X_j \in opa(V_i)}f_{i,j}(X_j) + \sum_{u_j \in upa(U_i)}f_{i,j}(U_j) + N_i,
\end{equation}
where $V=\{V_i\}$ is the set of observed or unobserved variables, $X=\{X_i\}$ the set of observed variables, $U=\{U_i\}$ is the set of unobserved variables, $N_i$ is the external effect on $V_i$, $opa(V_i)\subset X$ is the set of observed direct causes ({\it observed parents}) of $V_i$, $upa(V_i)\subset U$ is the set of unobserved direct causes ({\it unobserved parents}) of $V_i$, and $f_{i,j}$ is a nonlinear function. Additionally, Assumption~\ref{assumption:as1} is imposed on CAM-UV.

\begin{assumption}
	\label{assumption:as1}
All the causal functions and the external effects in CAM-UV satisfy the following condition: If variables $V_i$ and $V_j$ have terms involving functions of the same external effect $N_k$ , then $V_i$ and $V_j$ are mutually dependent (i.e., $(N_k\notindepe V_i)\land (N_k\notindepe V_j)\Rightarrow (V_i \notindepe V_j ) $). 
\end{assumption}

\subsection{TS-CAM-UV: Time series causal additive models with unobserved variables}\label{section:tsmodel}

Time-series causal additive noise models with unobserved variables (TS-CAM-UV) are stationary discrete-time structural causal models that can be described as below:

\begin{equation}
	\label{eq:3}
    V_i^t=\sum_{X_{j}^s \in opa(V_i^t)}f_{i,t,j,s}(X_{j}^s) + \sum_{U_{j}^s \in upa(U_i^t)}f_{i,t,j,s}(U_{j}^s) + N_i^t\ \ \ {\rm with}\ i=1,\dots,m,
\end{equation}
where $t$ and $s$ are time indices, $m$ is a natural number, $V=\{V_i^t\}$ is the set of observed or unobserved variables, $X=\{X_i^t\}$ is the set of observed variables, $U=\{U_i^t\}$ is the set of unobserved variables, $f_{i,t,j,s}$ is a nonlinear function, the noise variables $N_i^t$ are jointly independent, $opa(V_i^t)\subset X$ is the set of direct causes of $V_j^t$, and $upa(V_i^t)\subset U$ is the set of direct causes of $V_j^t$.\par
The stationarity of time-series causal relationships is assumed as the following: The causal relationship of the variable pair $(V^{t-\epsilon}_i,V^t_j)$ is the same as that of all the time shifted pairs $(V^{t^{\prime}-\epsilon}_i,V^{t^{\prime}}_j)$. The causal effect of $V^s_j$ on $V^t_i$ is called a lagged effect if $s < t$ holds, and is called a contemporaneous effect if $t=s$ holds. It is also assumed that there is a natural number $r$ as the maximum time lag such that the longest time lag of the direct causal effects does not exceed $r$. While it is true that the cause precedes the effect in time, if the time slice of the data analyzed are not sufficiently short, the cause and effect may appear to occur simultaneously. This type of causal effect, where the time difference between the cause and effect is shorter than the time slice of data, is referred to as a contemporaneous effect.

\begin{figure}[t]
\centering
\includegraphics[width=8.3cm]{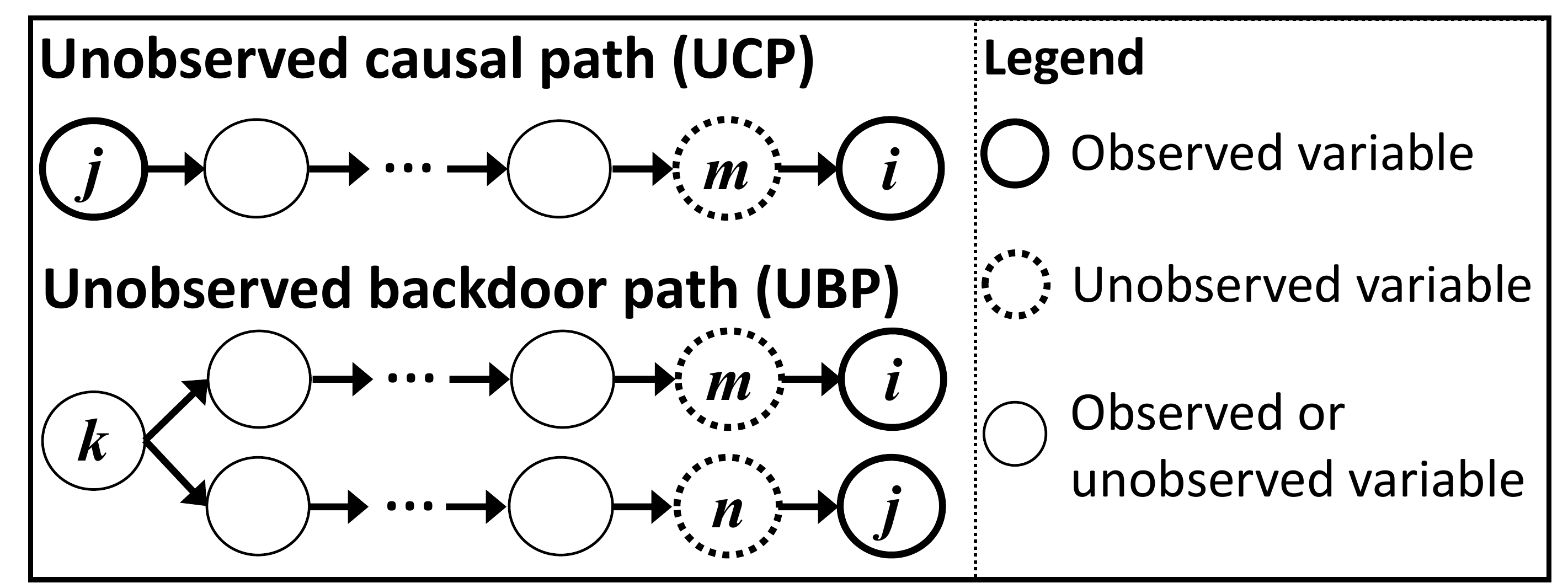}
\caption{Definitions of an unobserved causal path (UCP) and an unobserved backdoor path (UBP).}
\label{figure:path}
\end{figure}

\section{Identifiability}
\subsection{CAM-UV}
The identifiability of CAM-UV is discussed in \cite{pmlr-v161-maeda21a, maeda2021discovery}, and this section briefly presents it. When the causal relationship is linear, an observed variable $X_j$ being an indirect cause of an observed variable $X_i$, even if there is an unobserved variable $U_k$ in the causal path such that the causal relationship is $X_j\rightarrow U_k\rightarrow X_i$, the residual when regressing $X_i$ on $X_j$ becomes independent of $X_j$. However, in the case of a non-linear causal relationship, the residual when regressing $X_i$ on $X_j$ cannot be independent of $X_j$. This is referred to as cascade additive noise models (CANMs)~\cite{ijcai2019p223}. Therefore, in the case of non-linear causal relationships, compared to linear ones, there are more instances where causal relationships cannot be identified using only regression and independence tests. Before discussing the cases where causal relationships cannot be identified in CAM-UV, we define unobserved causal paths (UCPs) and unobserved backdoor paths (UBPs) which are illustrated in Fig.~\ref{figure:path} and used in the lemmas in this section.

\begin{defn}\label{defn:def1}
A directed path from an observed variable to another is called a {\it causal path} (CP). A CP from $X_j$ to $X_i$ is called an {\it unobserved causal path} (UCP) if it ends with the directed edge connecting $X_i$ and its unobserved direct cause (i.e., $X_j\rightarrow \cdots \rightarrow U_m\rightarrow X_i$ where $U_m$ is an unobserved direct cause of $X_i$).
\end{defn}

\begin{defn}\label{defn:def2}
An undirected path between $X_i$ and $X_j$ is called a {\it backdoor path} (BP) if it consists of the two directed paths from a common ancestor of $X_i$ and $X_j$ to $X_i$ and $X_j$ (i.e., $X_i\leftarrow \cdots \leftarrow V_k \rightarrow \cdots \rightarrow X_j$, where $V_k$ is the common ancestor). A BP between $X_i$ and $X_j$ is called an {\it unobserved backdoor path} (UBP) if it starts with the edge connecting $X_i$ and its unobserved direct cause, and ends with the edge connecting $X_j$ and its unobserved direct cause (i.e., $X_i\leftarrow U_m \leftarrow \cdots \leftarrow V_k \rightarrow \cdots \rightarrow U_n \rightarrow X_j$, where $V_k$ is the common ancestor and $U_m$ and $U_n$ are the unobserved direct causes of $X_i$ and $X_j$, respectively). The undirected path $X_i\leftarrow U_k \rightarrow X_j$ is also a UBP, as $V_k$, $U_m$, and $U_n$ can be the same variable.
\end{defn}

The identifiability of CAM-UV is based on Lemmas~1--3 shown below. They show that it is possible to identify the direct causal relationship between two variables if they do not have a UCP or a UBP, otherwise it is impossible to identify the direct direct causal relationship but possible to identify the presence of a UCP or a UBP. This is due to the fact that when the causal relationship is non-linear, if the parent of an observed variable $X_i$ is an unobserved variable $U_j$, the ancestral variables of $U_j$ cannot be removed from $X_i$ by regression. Lemma~1 is about the condition of variable pair $(X_i, X_j)$ having a UCP or a UBP. Lemma~2 is about the condition of variable pair $(X_i, X_j)$ not having a UBP, a UCP, or a direct causal relationship. Lemma~3 is about the condition that $X_j$ is a direct cause of $X_i$, and they do not have a UCP or a UBP. Please refer to \cite{maeda2021discovery} for the proofs of the lemmas.

\begin{lem}
\label{lem:lem1}
Assume the data generation process of the variables is CAM-UV as defined in Section~\ref{section:modelCAM-UV}. If and only if Equation~\ref{eq:lem1} is satisfied, there is a UCP or UBP between $X_i$ and $X_j$ where $G_1$ and $G_2$ denote regression functions satisfying Assumption~\ref{assumption:ass2}.
\begin{align}
\begin{aligned}
\label{eq:lem1}
	&\forall G_1, G_2, M_1 \subseteq (X \setminus \{X_i\}), M_2 \subseteq (X \setminus \{X_j\}),\\
	& \left[\left(X_i - G_1(M_1)\right) \notindepe \left(X_j-G_2(M_2)\right) \right] 
\end{aligned}
\end{align}
Equation~\ref{eq:lem1} indicates that the residual of $X_i$ regressed on any subset of $X\setminus\{X_i\}$ and the residual of $X_j$ regressed on any subset of $X\setminus\{X_j\}$ cannot be mutually independent.
\end{lem}

\begin{assumption}\label{assumption:ass2}
Let $M_1$ and $M_2$ denote sets satisfying $M_1\subseteq X$ and $M_2\subseteq X$ where $X$ is the set of all the observed variables in CAM-UV defined in Section~\ref{assumption:ass2}. We assume that functions $G_i$ take the forms of generalized additive models (GAMs)~\cite{hastie1990generalized} such that $G_i(M_1)=\sum_{X_m\in M_1}g_{i,m}(X_m)$ where each $g_{i,m}(X_m)$ is a nonlinear function of $X_m$. In addition, we assume that functions $G_i$ satisfy the following condition: When both $(X_i-G_i(M_1))$ and $(X_j-G_j(M_2))$ have terms involving functions of the same external effect $N_k$, then $(X_i-G_i(M_1))$ and $ (X_j-G_j(M_2))$ are mutually dependent  (i.e., $(N_k\notindepe X_i-G_i(M_1))\land (N_k\notindepe X_j-G_j(M_2))\Rightarrow ((X_i-G_i(M_1)) \notindepe (X_j-G_j(M_2)) ) $).\end{assumption}

\begin{lem}
\label{lem:lem2}
Assume the data generation process of the variables is CAM-UV as defined in Section~\ref{section:modelCAM-UV}. If and only if Equation~\ref{eq:lem2} is satisfied, there is no direct causal relationship between $X_i$ and $X_j$, and there is no UCP or UBP between $X_i$ and $X_j$ where $G_1$ and $G_2$ denote regression functions satisfying Assumption~\ref{assumption:ass2}.
\begin{align}
\begin{aligned}
\label{eq:lem2}
	&\exists G_1, G_2, M \subseteq (X \setminus \{X_i,X_j\}), N \subseteq (X \setminus \{X_i,X_j\}),\\
	&[(\left(X_i - G_1(M)\right) \indepe \left(X_j-G_2(N)\right))] 
\end{aligned}
\end{align}
Equation~\ref{eq:lem2} indicates that there are regression functions such that the residuals of $X_i$ and $X_j$ regressed on subsets of $X\setminus\{X_i,X_j\}$ are mutually independent.
\end{lem}

\begin{lem}
\label{lem:lem3}
Assume the data generation process of the variables is CAM-UV as defined in Section~\ref{section:modelCAM-UV}. If and only if Equations~\ref{eq:lem3-1} and \ref{eq:lem3-2} are satisfied, $X_j$ is a direct cause of $X_i$, and there is no UCP or UBP between $X_i$ and $X_j$ where $G_1$ and $G_2$ denote regression functions satisfying Assumption~\ref{assumption:ass2}.
\begin{align}
\begin{aligned}
\label{eq:lem3-1}
	&\forall G_1, G_2, M \subseteq (X \setminus \{X_i,X_j\}), N \subseteq (X \setminus \{X_j\}),\\
	& \left[\left(X_i - G_1(M)\right) \notindepe \left(X_j-G_2(N)\right) \right] 
\end{aligned}
\end{align}
\begin{align}
\begin{aligned}
\label{eq:lem3-2}
	&\exists G_1, G_2, M \subseteq (X \setminus \{X_i\}), N \subseteq (X \setminus \{X_i,X_j\}),\\
	& \left[\left(X_i - G_1(M)\right) \indepe \left(X_j-G_2(N)\right) \right] 
\end{aligned}
\end{align}
Equation~\ref{eq:lem3-1} indicates that the residual of $X_i$ regressed on any subset of $X\setminus\{X_i,X_j\}$ and the residual of $X_j$ regressed on any subset of $X\setminus\{X_j\}$ cannot be mutually independent. Equation~\ref{eq:lem3-2} indicates that there are regression functions such that the residual of $X_i$ regressed on a subset of $X\setminus\{X_j\}$ and the residual of $X_j$ regressed on a subset of $X\setminus\{X_i,X_j\}$ are mutually independent.

\end{lem}

\subsection{TS-CAM-UV}

The identifiability of causality in TS-CAM-UV is the same as in CAM-UV. Lemmas 4--6 on identifiability in TS-CAM-UV correspond to Lemmas 1--3 on identifiability in CAM-UV.

\begin{lem}
\label{lem:lemts1}
Assume the data generation process of the variables is TS-CAM-UV as defined in Section~\ref{section:tsmodel}. If and only if Equation~\ref{eq:lem-ts1} is satisfied, there is a UCP or UBP between $X_i^t$ and $X_j^s$ where $G_1$ and $G_2$ denote regression functions satisfying Assumption~\ref{assumption:ass2}.
\begin{align}
\begin{aligned}
\label{eq:lem-ts1}
	&\forall G_1, G_2, M \subseteq (X \setminus \{X_i^t\}), N \subseteq (X \setminus \{X_j^s\}),\\
	& \left[\left(X_i^t - G_1(M)\right) \notindepe \left(X_j^s-G_2(N)\right) \right] 
\end{aligned}
\end{align}
\end{lem}
\begin{proof}
	The relationships between $X_i^t$ and $X_j^s$ in TS-CAM-UV are the same as those of $X_i$ and $X_j$ in CAM-UV defined in Section~\ref{section:modelCAM-UV}. Therefore, Lemma~\ref{lem:lemts1} holds because of Lemma~\ref{lem:lem1}.
\end{proof}

\begin{lem}
\label{lem:lem-ts2}
Assume the data generation process of the variables is TS-CAM-UV as defined in Section~\ref{section:tsmodel}. If and only if Equation~\ref{eq:lem-ts2} is satisfied, there is no direct causal relationship between $X_i^t$ and $X_j^s$, and there is no UCP or UBP between $X_i^t$ and $X_j^s$ where $G_1$ and $G_2$ denote regression functions satisfying Assumption~\ref{assumption:ass2}.
\begin{align}
\begin{aligned}
\label{eq:lem-ts2}
	&\exists G_1, G_2, M \subseteq (X \setminus \{X_i^t,X_j^s\}), N \subseteq (X \setminus \{X_i^t,X_j^s\}),\\
	&[(\left(X_i^t - G_1(M)\right) \indepe \left(X_j^s-G_2(N)\right))] 
\end{aligned}
\end{align}
\end{lem}
\begin{proof}
	The relationships between $X_i^t$ and $X_j^s$ in TS-CAM-UV are the same as those of $X_i$ and $X_j$ in CAM-UV defined in Section~\ref{section:modelCAM-UV}. Therefore, Lemma~\ref{lem:lem-ts2} holds because of Lemma~\ref{lem:lem2}.
\end{proof}

\begin{lem}
\label{lem:lem-ts3}
Assume the data generation process of the variables is TS-CAM-UV as defined in Section~\ref{section:tsmodel}. If and only if Equations~\ref{eq:lem-ts3-1} and \ref{eq:lem-ts3-2} are satisfied, $X_j^s$ is a direct cause of $X_i^t$, and there is no UCP or UBP between $X_i^t$ and $X_j^s$ where $G_1$ and $G_2$ denote regression functions satisfying Assumption~\ref{assumption:ass2}.
\begin{align}
\begin{aligned}
\label{eq:lem-ts3-1}
	&\forall G_1, G_2, M \subseteq (X \setminus \{X_i^t,X_j^s\}), N \subseteq (X \setminus \{X_j^t\}),\\
	& \left[\left(X_i^t - G_1(M)\right) \notindepe \left(X_j^s-G_2(N)\right) \right] 
\end{aligned}
\end{align}
\begin{align}
\begin{aligned}
\label{eq:lem-ts3-2}
	&\exists G_1, G_2, M \subseteq (X \setminus \{X_i^t\}), N \subseteq (X \setminus \{X_i^t,X_j^s\}),\\
	& \left[\left(X_i^t - G_1(M)\right) \indepe \left(X_j^s-G_2(N)\right) \right] 
\end{aligned}
\end{align}
\end{lem}
\begin{proof}
	The relationships between $X_i^t$ and $X_j^s$ in TS-CAM-UV are the same as those of $X_i$ and $X_j$ in CAM-UV defined in Section~\ref{section:modelCAM-UV}. Therefore, Lemma~\ref{lem:lem-ts3} holds because of Lemma~\ref{lem:lem3}.
\end{proof}

\section{Methods}
\subsection{CAM-UV-PK: Causal additive models with unobserved variables using prior knowledge}

\begin{figure}[t]
\centering
\includegraphics[width=10.3cm]{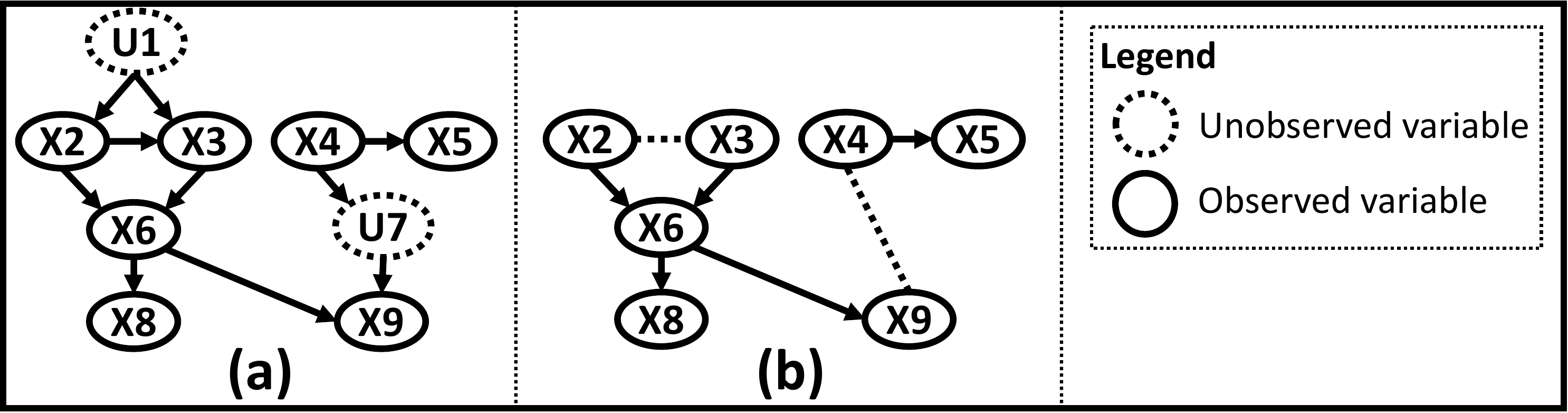}
\caption{(a): True causal graph. (b): Causal graph generated by the CAM-UV algorithm.}
\label{figure:camuvgraph}
\end{figure}

This section proposes a method called CAM-UV using prior knowledge (CAM-UV-PK). This method is for discovering causal additive models with unobserved models defined in Section~\ref{section:modelCAM-UV}. In addition to the arguments of the CAM-UV algorithm, the CAM-UV-PK algorithm requires an argument ${\mathbf T}$ that is a list of ordered variable pairs. If a ordered variable pair $(X_i, X_j)$ is included in ${\mathbf T}$, it means that it is assumed that $X_i$ cannot be a direct or indirect cause of $X_j$. 

The CAM-UV algorithm and CAM-UV-PK algorithm output causal graphs with directed edges and undirected dashed edges. Directed edges indicate variable pairs having direct causal relationships, and undirected dashed edges indicate variable pairs having UCPs or UBPs. For example, Figure~\ref{figure:camuvgraph}-(a) shows a true causal graph, and Figure~\ref{figure:camuvgraph}-(b) shows the causal graph generated by the CAM-UV algorithm. $X_2$ and $X_3$ have a UBP ($X_2\leftarrow U_1 \rightarrow X_3$), so they are connected with an undirected dashed path in Figure~\ref{figure:camuvgraph}-(b). $X_4$ and $X_9$ have a UCP ($X_4\rightarrow U_7 \rightarrow X_9$), so they are also connected with an undirected dashed path in Figure~\ref{figure:camuvgraph}-(b).

\begin{algorithm}[]
\SetKwProg{init}{initialization}{}{}
\DontPrintSemicolon
\KwIn{${\mathbf X}$: Samples of $p$ observed variables $\{X_1,\cdots,X_p\}$, ${\mathbf T}$: A list of ordered pairs of variables where it is assumed that the first variable cannot be the direct or indirect cause of the second variable, $d$: maximal number of variables to examine causality for each step, significance level for independence test $\alpha$.}
\KwOut{the sets of the parents $\{M_1, \cdots, M_p \}$.}
\SetKwBlock{Begin}{function}{end function}
\Begin($\text{getDirectedEdges} {(} X, d, \alpha {)}$)
{
	{\sf\# PHASE 1: Extracting the candidates of the parents of each variable.}\;
	    \For{$i=1$ {\bf to} $p$}{
    	{\bf Initialize} $M_i \leftarrow \emptyset$.\;
    }
    {\bf Initialize} $t \leftarrow 2$.\;
    \While{$t\leq d$}{
    	{\bf Initialize} $noChange \leftarrow {\rm True}$.\;
    	\ForEach{$K\in\{K|K\subseteq X, |K|=t\}$}{
			{\sf\# Finding the most endogenous variable $X_b$ in {\it K}}\;
			\ForEach{$K\in\{K|K\subseteq X, |K|=t\}$}{
				$maxIndependence\leftarrow 0$\;
			$maxIndependenceVariable\leftarrow {\rm NULL}$\;
				\ForEach{$X_i\in K$}{
					\sf\# Checking whether there exists variable $X_j \in K \setminus \{X_i\}$ that cannot be a cause of $X_i$ according to prior knowledge ${\mathbf T}$.\;
					\If{$\exists X_j\in K\setminus \{X_i\},\ \ [(X_j, X_i) \in {\bf T}]$}{
						continue\;
					}
					$indepe\leftarrow \reallywidehat{\text{p-HSIC}}(X_i-G_1(M_i\cup K\setminus\{X_i\}),\{X_j-G_2(M_j)| X_j\in K\setminus\{X_i\}\})$\;
					\If{$maxIndependence<indepe$}{
							$maxIndependence\leftarrow indepe$\;
								$maxIndependenceVariable\leftarrow X_i$\;
					}
				}
				$X_b\leftarrow maxIndependenceVariable$\;
			}
			{\sf\# Computing the independence between the residuals}\;
			$e\leftarrow\reallywidehat{\text{p-HSIC}}(X_b-G_1(M_b\cup K\setminus\{X_b\}),\{X_j-G_2(M_j)| X_j\in K\setminus\{X_b\}\})$\;
			$h\leftarrow\displaystyle\max_{x_j\in K\setminus\{X_b\}}\reallywidehat{\text{p-HSIC}}(X_b-G_1(M_b), X_j-G_2(M_j))$\;
			{\sf\# Checking whether $X_b$ is really a sink of {\it $K$}}\;
			\If{$(\alpha < e)\land(\alpha > h)$}{
				{\sf\# When $X_b$ is a sink of $K$, add each variable in $K\setminus\{X_b\}$ to $M_b$}\;
				$M_b \leftarrow M_b\cup (K\setminus\{X_b\})$\;
				$noChange \leftarrow {\rm False}$\;
			}
    	}
    	{\sf\# If each $M_i$ remains unchanged, increment $t$ by one. If not, substitute $2$ for $t$.}\;
    	\uIf{$noChange={\rm True}$}{
    		$t\leftarrow t+1$\;
    	}\Else{
    		$t\leftarrow 2$\; 
    	}
    }
    {\sf\# PHASE 2: Determining the parents of each variable.}\;
    \For{$i=1$ {\bfseries to} $p$}{
    	\ForEach{$X_j\in M_i$}{
    		{\sf\# Checking whether $X_j$ is parent of $X_i$}\;
    		\If{$\alpha < \reallywidehat{\text{\rm p-HSIC}}(X_i-G_1(M_i\setminus \{X_j\}),X_j-G_2(M_j))$}{
	    		{\sf \# When $X_j$ is not a parent, remove it from $M_i$}\;
    			$M_i\leftarrow M_i\setminus \{X_j\}$\;
    		}
    	}
    }
    \Return{$\{M_1, \cdots, M_p \}$}
}
\caption{Determine the directed edges}
   \label{algo:algorithm}
\end{algorithm}

The CAM-UV-PK algorithm incorporates restriction using prior knowledge ${\mathbf T}$ into the process of causal inference in the CAM-UV algorithm. The CAM-UV algorithm has two-step algorithm~\cite{pmlr-v161-maeda21a, maeda2021discovery}. The first step determines the directed edges, and the second one determines the undirected dashed edges. There is no difference in the second step between the CAM-UV-PK algorithm and the CAM-UV algorithm. The first step of the CAM-UV-PK algorithm is listed in Algorithm~\ref{algo:algorithm}. Lines~14--16 in Algorithm~\ref{algo:algorithm} are added to the CAM-UV algorithm. This part of the algorithm refers to the prior knowledge ${\mathbf T}$ to avoid considering unnecessary causal candidates. The method extracts the candidates of the direct causes (parents) of each variable (lines~2--34) and determines the direct causes of each variables (lines~35--41). The method identifies the most endogenous variable $X_b$ in each $K\in\{K|K\subseteq X, |K|=t\}$. When $X_i=x_b$ is satisfied, $X_i$ maximizes $\reallywidehat{\text{p-HSIC}}(X_i-G_1(M_i\cup K\setminus\{X_i\}),\{X_j-G_2(M_j)| X_j\in K\setminus\{X_i\}\})$. In lines~14--16 which are newly added in CAM-UV-PK, the method checks whether there exists $X_j\in K \setminus\{X_i\}$ that cannot be a direct or indirect cause of $X_i$ according to the prior knowledge ${\mathbf T}$. If $(X_j, X_i)\in {\mathbf T}$ is satisfied, the method stops checking whether $X_i$ is endogenous to $K\setminus\{X_i\}$. Therefore, this check prevents incorrect inference of causal relationships.

\subsection{TS-CAM-UV: Time series causal additive models with unobserved variables}

This section proposes a method called the time-series CAM-UV (TS-CAM-UV) algorithm. The TS-CAM-UV algorithm uses as prior knowledge the assumption, called time priority, that effect does not precede its cause in time. The TS-CAM-UV algorithm uses the CAM-UV-PK algorithm, and the prior knowledge of time priority is used for the argument of the CAM-UV-PK, ${\mathbf T}$.

The TS-CAM-UV algorithm first creates data with $q\times (r+1)$ variables where $q$ is the number of the variables of original data, and $r$ is the maximal considered time lag given as an argument. Let ${\mathbf X_t}=\{X^t_1,\cdots,X^t_q\}$ denote the variables in original data. The TS-CAM-UV algorithm creates data with variables $\mathbf X^{\rm new}=\{X^t_1,\cdots,X^t_q,X^{t-1}_1,\cdots,X^{t-1}_q,\cdots,X^{t-r}_1,\cdots,X^{t-r}_q\}$. Therefore, the number of the new samples decreases to $n-r$ where $n$ denotes the number of original data. Then, the TS-CAM-UV algorithm creates a list of ordered variables $K=\{(X^t_i,X^{t^{\prime}}_j)|t>t^{\prime}, 0\leq i\leq q, 0\leq j\leq q, \}$.\par
The TS-CAM-UV algorithm uses $\mathbf X^{\rm new}$ and $K$ for the arguments of CAM-UV-PK ${\mathbf X}$ and ${\mathbf T}$, respectively. Then, CAM-UV-PK outputs a causal graph of the $q$ variables with $r$ time lag.

\section{Experiments}
We conducted experiments to examine the performance of the CAM-UV-PK algorithm and the TS-CAM-UV algorithm. The CAM-UV-PK algorithm is compared with that of CAM-UV. The TS-CAM-UV algorithm is compared with VarLiNGAM and LPCMCI. Here, we primarily compare the accuracy of directed edges. This is because, in other methods, there are no approaches that consider the effects of unobserved intermediate variables (unobserved variables on the causal paths between observed variables), and also because CAM-UV aims to ensure that the inference of directed edges is not biased due to latent confounders.

\subsection{CAM-UV-PK: Causal additive models with unobserved variables using prior knowledge}

We examined the performance of CAM-UV-PK compared to CAM-UV using simulated data. We compared and evaluated the performance of CAM-UV-PK with prior knowledge ranging from 0 to 4. The CAM-UV algorithm is the same as the CAM-UV-PK algorithm with no input of prior knowledge. We performed 100 experiments using artificial data with each sample size $n\in\{100, 200, \cdots, 900, 1000\}$ to compare our method to existing methods. In each experiment, the samples are created as follows:
\begin{itemize}
	\item The number of observed variables is 10.
	\item The number of the observed variable pairs having unobserved common causes is 4.
	\item The number of observed variable pairs having unobserved causal intermediate variables is 2.
	\item The number of the observed variable pairs having direct causal effects is 10.
	\item Variable pairs having unobserved common causes, unobserved intermediate causal variables, or direct causal relationships were randomly selected under the restriction that the set of variable pairs with unobserved common causes, the set of variable pairs with unobserved intermediate causal variables, and the set of variable pairs with direct causal relationships were mutually disjoint.
	\item The causal effect of $V_{j}$ on $V_{i}$ is determined as follows:
\begin{equation}
\left(\sin\left(a_1 \left( V_j+b_1\right) \right)\right)^3 c_1+\left(\frac{1}{1+\exp(-a_2(V_j+b_2))}-0.5\right)c_2
\end{equation}
where $a_1$, $a_2$, $b_1$, $b_2$, $c_1$, and $c_1$ are constants that take random value for each $(i,j)$. Constants $a_1$ and $a_2$ are taken from $U(9,11)$, $b_1$ and $b_2$ are taken from $U(-0.1,0.1)$, and $c_1$ and $c_2$ are taken from $U(3,5)$. This function is also used in experiments to validate the TS-CAM-UV algorithm in the next section so that causal effects do not converge or diverge over time.
\end{itemize}

The arguments of TS-CAM-UV, $\alpha$ (significance level for independence test) and $d$ (maximal number of variables to examine causality for each step) are set to 0.01 and 2, respectively.

We compared the performance of the identification of direct causal relationships. We used precision, recall, and F-measure as the evaluation measures. True positive (TP) is the number of true directed edges that a method correctly infers in terms of their positions and directions. Precision represents the TP divided by the number of estimations, and recall represents the TP divided by the number of all the true directed edges. Furthermore, F-measure is defined as $\text{F-measure} = 2 \cdot \text{precision} \cdot \text{recall} / (\text{precision} + \text{recall})$. In each experiment, out of the ten variable pairs with direct causal relationships, four were excluded from the evaluation. These four causal relationships were used as prior knowledge in CAM-UV-PK.

\begin{figure}[t]
\centering
\includegraphics[width=12.0cm]{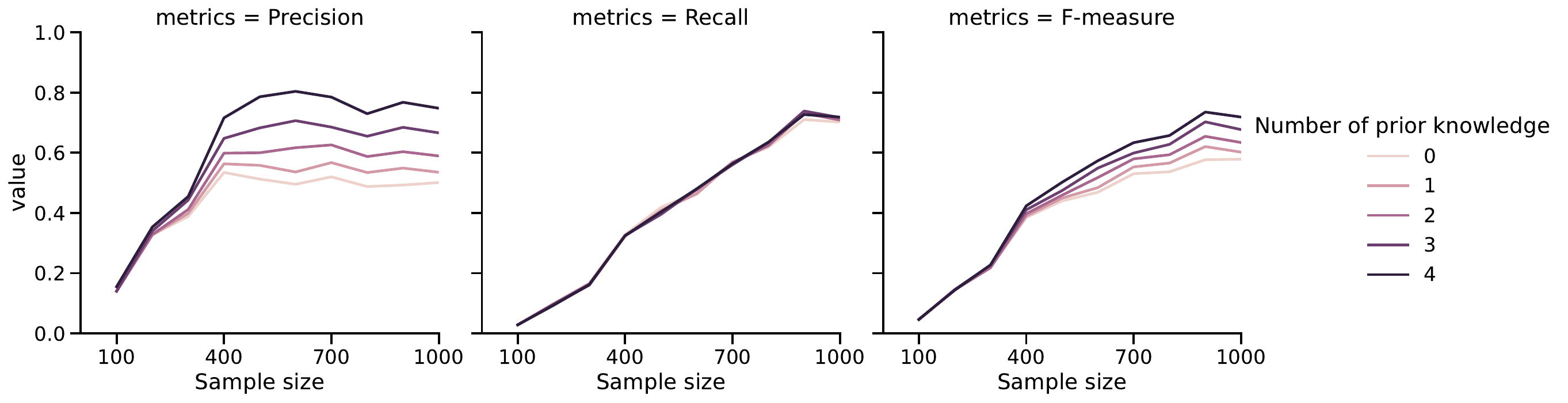}
\caption{The performance of the CAM-UV-PK and CAM-UV algorithms: The CAM-UV algorithm is equivalent to the CAM-UV-PK algorithm with no prior knowledge.}
\label{figure:res_CAM-UVpr}
\end{figure}

Fig.~\ref{figure:res_CAM-UVpr} shows the results of the identification of direct causal relationships. The figure plots the average of precision, recall, and F-measure. It can be seen that precision and F-measure increase with the number of prior knowledge. The CAM-UV algorithm is the CAM-UV-PK algorithm without prior knowledge, and this has the lowest precision and F-measure. The number of prior knowledge does not significantly affect recall.

The above experimental results of the CAM-UV-PK algorithm confirm that the number of prior knowledge improves the precision and F-measure of the identification of direct causal relationships.

\subsection{TS-CAM-UV: Time series causal additive models with unobserved variables}

We examined the performance of TS-CAM-UV compared to LPCMCI and VarLiNGAM using simulated data and real-world data. For LPCMCI, two methods of conditional independence test were used for the comparison: Partial correlation test (ParCorr) and Gaussian process regression and a distance correlation test on the residuals (GPDC). ParCorr assumes linear additive noise models, and GPDC assumes nonlinear additive noise models.

\subsubsection{Simulated data}
\label{sec:simutest}

We performed 100 experiments using artificial data with each sample size $n\in\{100, 200, \cdots, 1900, 2000\}$ to compare our method to existing methods. In each experiment, the samples are created as follows:
\begin{itemize}
	\item The number of observed variables and the maximum time lag are 3 and 2, respectively. Therefore, the number of the variables representing different time lags of all the observed variables is 9 (e.g. $|\{X_i^t\}|=9$).
	\item The number of observed variable pairs having unobserved common causes is 2.
	\item The number of observed variable pairs having unobserved intermediate variables is 2.
	\item The number of observed variable pairs having direct causal relationships is 5.
	\item Variable pairs having unobserved common causes, unobserved intermediate causal variables, or direct causal relationships were randomly selected under the restriction that the set of variable pairs with unobserved common causes, the set of variable pairs with unobserved intermediate causal variables, and the set of variable pairs with direct causal relationships were mutually disjoint.
	\item The causal effect of $V^{s}_j$ on $V^{t}_i$ is determined as below:
\begin{equation}
\left(\sin\left(a_1 \left( V_{s}^j+b_1\right) \right)\right)^3 c_1+\left(\frac{1}{1+\exp(-a_2(V_{s}^j+b_2))}-0.5\right)c_2
\end{equation}
where $a_1$, $a_2$, $b_1$, $b_2$, $c_1$, and $c_1$ are constants that take random value for each $(i,j,t,t^{\prime})$. Constants $a_1$ and $a_2$ are taken from $U(9,11)$, $b_1$ and $b_2$ are taken from $U(-0.1,0.1)$, and $c_1$ and $c_2$ are taken from $U(3,5)$.
\end{itemize}

In this experiment, we compared the performance of the identification of direct causal relationships. That is, the edges with arrows in causal graphs ($\rightarrow$).

The arguments of the TS-CAM-UV algorithm, VarLiNGAM, and LPCMCI were set as follows:
\begin{itemize}
	\item TS-CAM-UV
	\begin{itemize}
		\item[$\circ$] Significance level for independence test: 0.01.
		\item[$\circ$] Maximal number of causal variables to examine causality for each step: 2.
		\item[$\circ$] Maximal number of time lags: 2.
	\end{itemize}
	\item VarLiNGAM
	\begin{itemize}
		\item[$\circ$] Maximal number of time lags: 2.
		\item[$\circ$] Threshold value for the strength of the causal effects (i.e. the absolute values of coefficients): 0.01, 0.05, 0.1, and 0.5.
	\end{itemize}
	\item LPCMCI
	\begin{itemize}
		\item[$\circ$] Significance level for independence test: 0.01.
		\item[$\circ$] Maximal number of time lags: 2.
		\item[$\circ$] Methods of conditional independence test: GPDC and ParCorr.
	\end{itemize}
\end{itemize}

The results are shown in Fig.~\ref{figure:ts_CAM-UV}. The figure plots the average of precision, recall, and F-measure. The values in the brackets for VarLiNGAM indicate threshold values for the strength of causal effects. TS-CAM-UV showed the highest precision for $n\geq 200$, the highest recall for $n\geq 1200$, and the highest F-measure for $n\geq 600$ compared to other methods.

\begin{figure}[t]
\centering
\includegraphics[width=12.0cm]{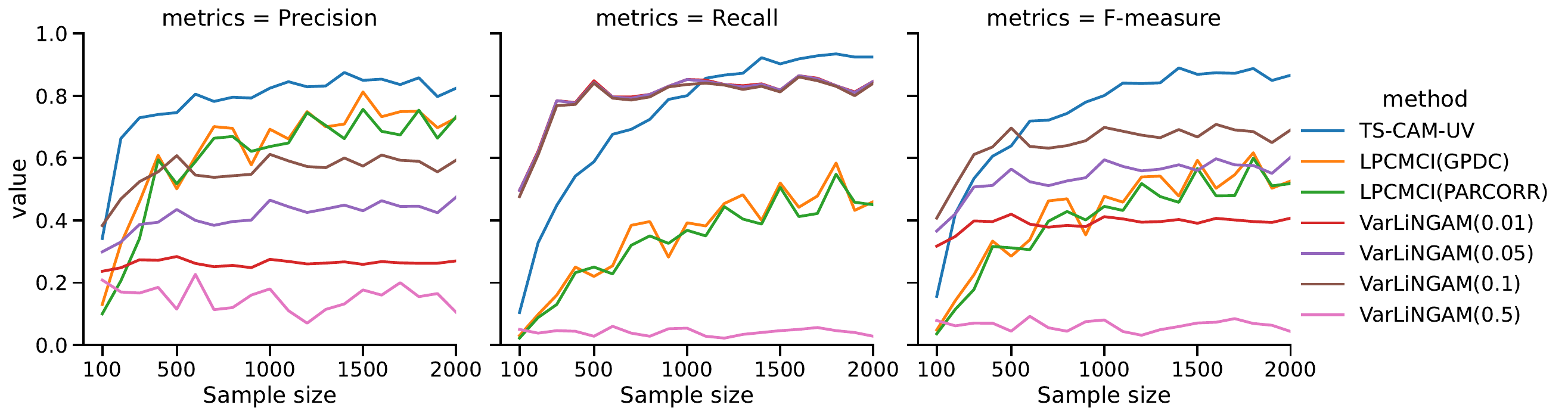}
\caption{The performance of the TS-CAM-UV compared to LPCMCI and VarLiNGAM.}
\label{figure:ts_CAM-UV}
\end{figure}

\subsubsection{Real world data}

We also conducted an experiment using official foreign exchange quotation data for the Japanese yen at Mizuho Bank\footnote{Mizuho Bank: https://www.mizuhobank.co.jp/market/historical.html (in Japanese).}. The data consist of daily quotes for USD, GBP, EUR, CHF, and CAD from the 26th October 2021 to the 8th November 2023. The total sample size is 500.

We set the maximal lag length of every method to 1. The threshold value for causal effects for VarLiNGAM was set to 0.1 which gave the best result in experiments using simulated data shown in Section~\ref{sec:simutest}. All other arguments were kept the same as in Section~\ref{sec:simutest}.

\begin{figure}[h]
  \begin{minipage}[b]{.49\linewidth}
    \centering
    \includegraphics[keepaspectratio, scale=0.23]{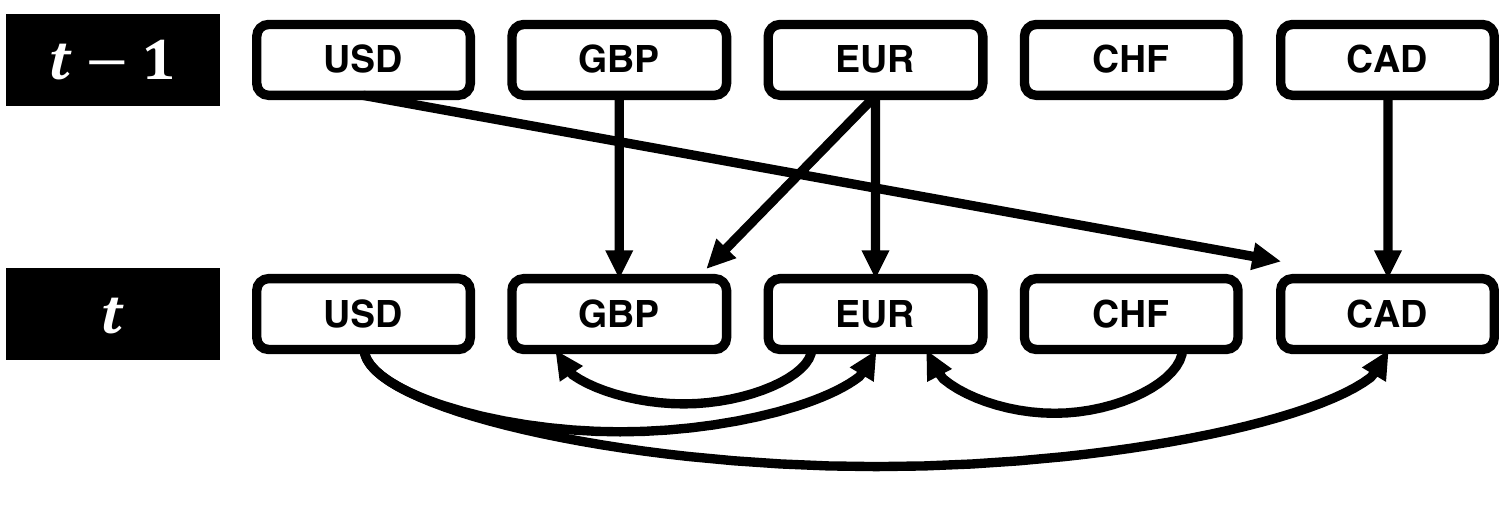}
    \subcaption{The causal graph only with directed edges generated from TS-CAM-UV.\\ }
  \end{minipage}
  \ \ \ \begin{minipage}[b]{.49\linewidth}
    \centering
    \includegraphics[keepaspectratio, scale=0.23]{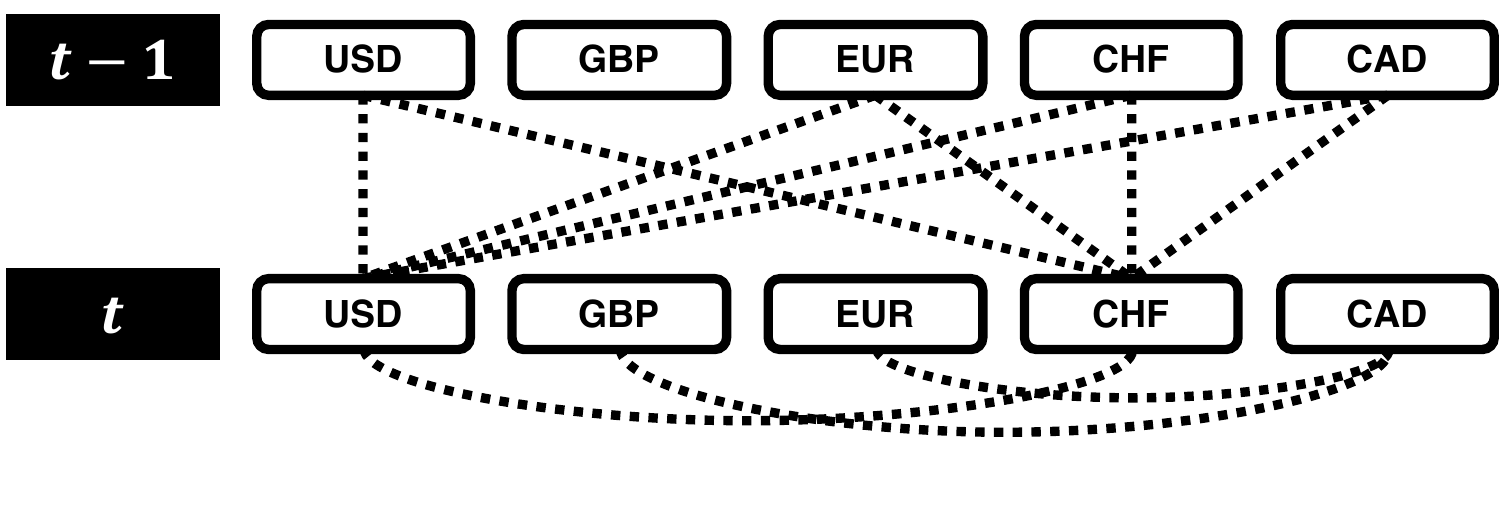}
    \subcaption{The causal graph only with edges other than directed edges generated from TS-CAM-UV.}
  \end{minipage}\\ \\ 
  \begin{minipage}[b]{.49\linewidth}
    \centering
    \includegraphics[keepaspectratio, scale=0.23]{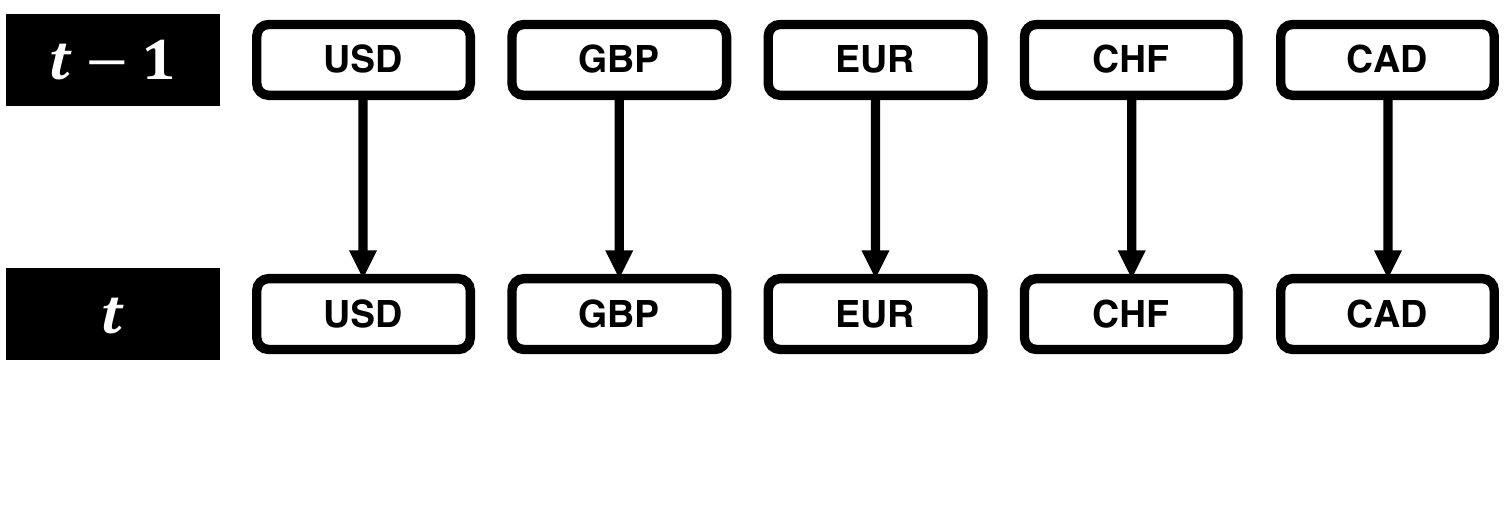}
    \subcaption{The causal graph only with directed edges generated from LPCMCI (ParCorr).\\}
  \end{minipage}
  \ \ \ \begin{minipage}[b]{.49\linewidth}
    \centering
    \includegraphics[keepaspectratio, scale=0.23]{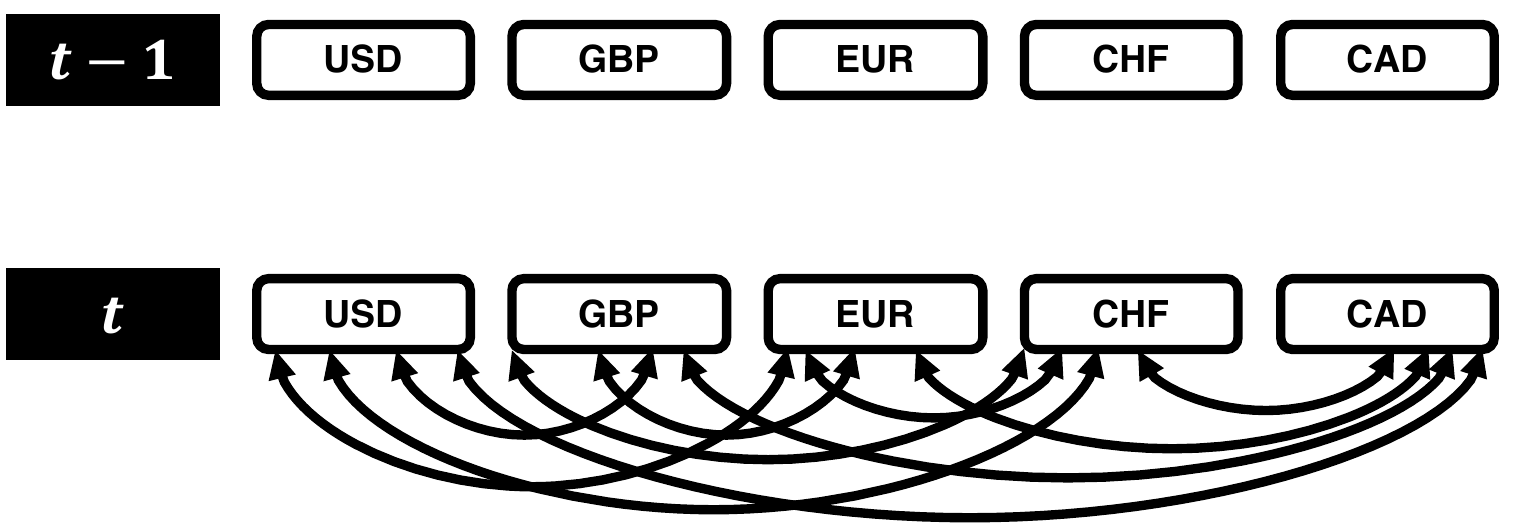}
    \subcaption{The causal graph with edges other than directed edges generated from LPCMCI (ParCorr).}
  \end{minipage}\\ \\ 
  \begin{minipage}[b]{.49\linewidth}
    \centering
    \includegraphics[keepaspectratio, scale=0.23]{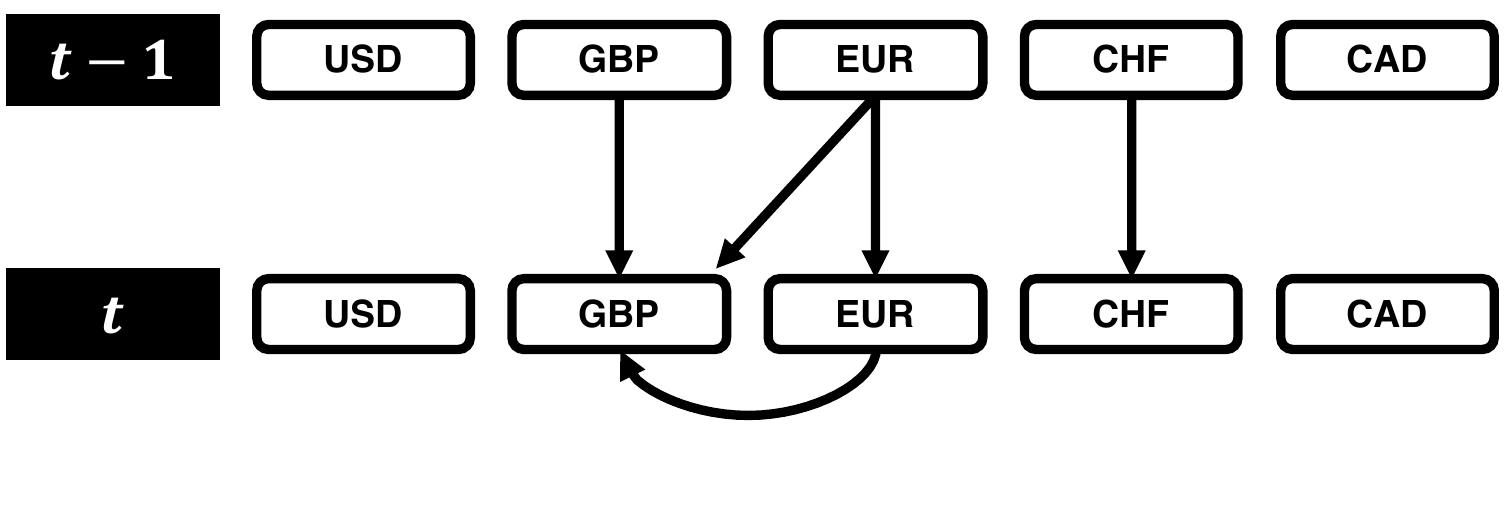}
    \subcaption{The causal graph only with directed edges generated from LPCMCI (GPDC).\\}
  \end{minipage}
  \ \ \ \begin{minipage}[b]{.49\linewidth}
    \centering
    \includegraphics[keepaspectratio, scale=0.23]{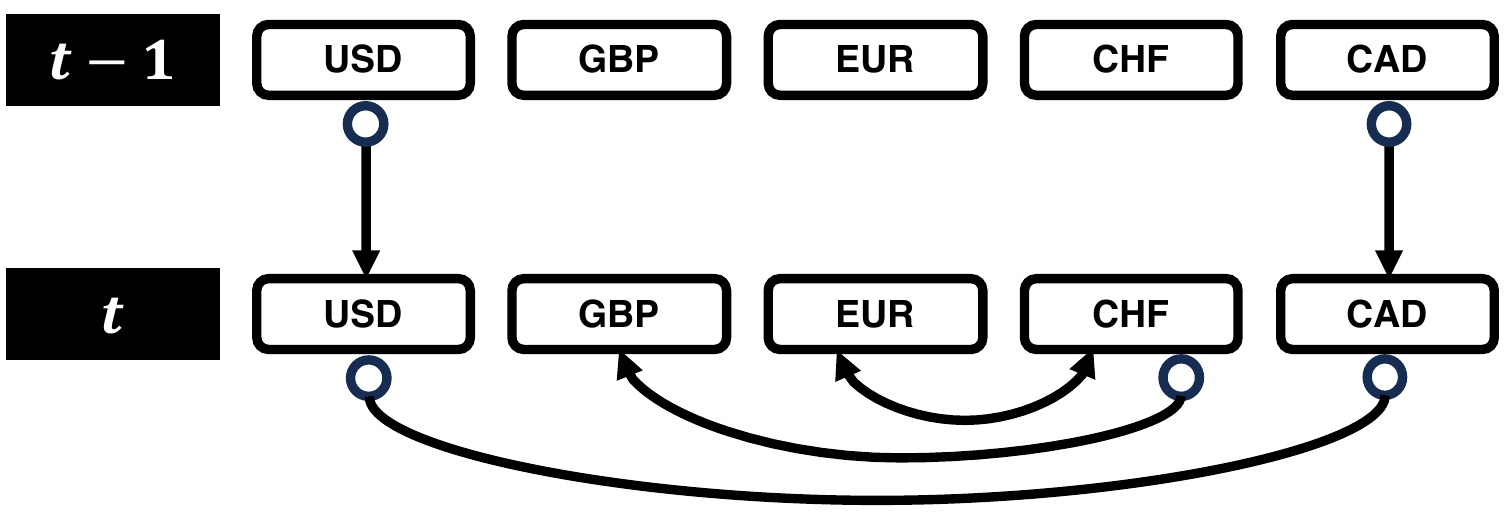}
    \subcaption{The causal graph with edges other than directed edges generated from LPCMCI (GPDC).}
  \end{minipage}\\ \\ 
  \begin{minipage}[b]{.49\linewidth}
    \centering
    \includegraphics[keepaspectratio, scale=0.23]{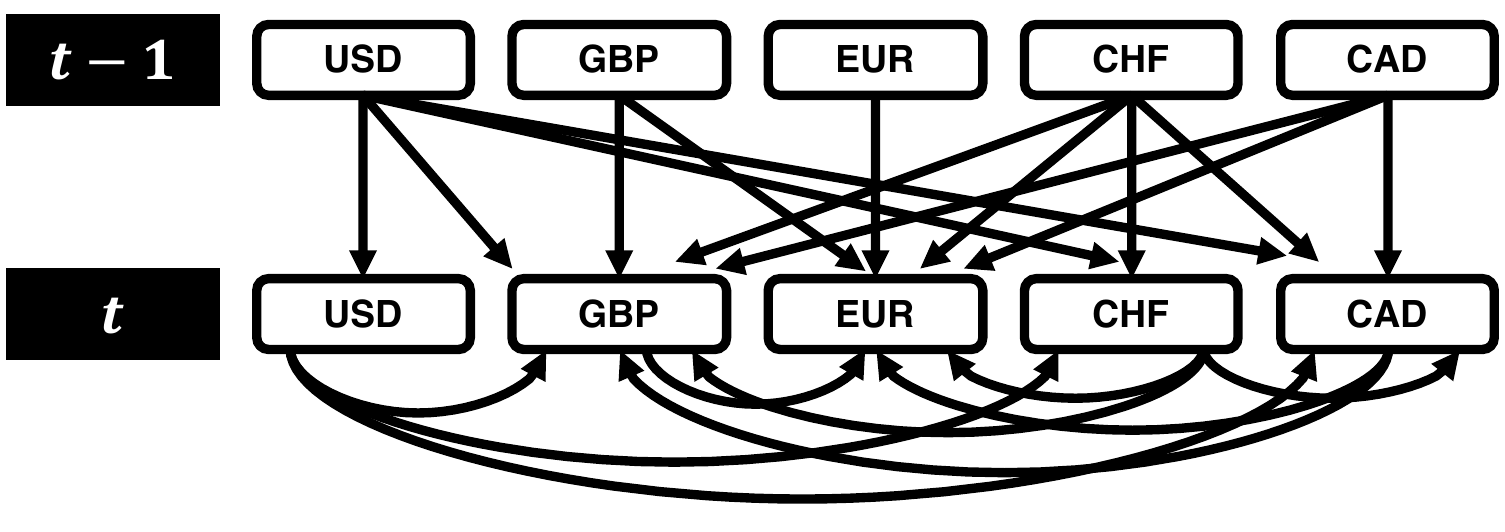}
    \subcaption{The causal graph generated from TS-VarLiNGAM.}
  \end{minipage}
\caption{Causal graphs generated using foreign exchange data.}
\label{figure:realresult}
\end{figure}

Fig.~\ref{figure:realresult} shows the results: (a) The causal graph only with directed edges generated from TS-CAM-UV, (b) the causal graph with edges other than directed edges generated from TS-CAM-UV, (c) the causal graph only with directed edges generated from LPCMCI using ParCorr, (d) the causal graph with edges other than directed edges generated from LPCMCI using ParCorr, (e) the causal graph only with directed edges generated from LPCMCI using GPDC, (f) the causal graph with edges other than directed edges generated from LPCMCI using GPDC, and (g) the causal graph only with directed edges generated from VarLiNGAM. The dashed lines in Fig.~\ref{figure:realresult}-(b) show the variable pairs estimated to have UBPs or UCPs.

We do not compare the performance of the methods based on the results because there is no ground truth for the relationships among the variables. We compare TS-CAM-UV with other methods to see what kind and how many variable pairs are connected. Figure~\ref{figure:realresult}-(c) shows that LPCMCI (ParCorr) draws an edge from the variable representing the state at time $t-1$ of each currency to the variable representing the state at time $t$ of the currency (e.g. $X_{t-1}^i\rightarrow X_{t}^i$), but does not draw edges between variables of different currencies. Compared to this, Figure~\ref{figure:realresult}-(a) shows that TS-CAM-UV connects the variables of different currencies with directed edges. This may be due to the fact that ParCorr assumes linear causal relationships. The causal relationship between the previous and current values of the same currency may be linear, while other causal relationships may be nonlinear. Figure~\ref{figure:realresult}-(e) shows LPCMCI (GPDC) connects the variables of different currencies with directed edges. The number of variable pairs connected by LPCMCI (GPDC) is less than the number of variable pairs connected by TS-CAM-UV. This may be due to the fact that LPCMCI is a constraint-based method and cannot distinguish between all graphs with the same set of conditional independence between observed variables.
Figure~\ref{figure:realresult}-(g) shows that VarLiNGAM connects more variable pairs with directed edges than TS-CAM-UV. This may be due to the fact that VarLiNGAM assumes the absence of latent confounders.

To summarize, the TS-CAM-UV algorithm is based on a causal functional model, which enables it to identify the direction of causality in variable pairs that LPCMCI could not orient. Furthermore, by assuming the presence of unobserved variables, it can avoid incorrect orientations, similar to what occurs with VarLiNGAM.

\section{Conclusion}
In this paper, we propose two methods as extensions of CAM-UV: CAM-UV-PK and TS-CAM-UV. The CAM-UV-PK algorithm employs a method that introduces prior knowledge in the form that a certain variable is not a cause of a certain other variable. This is based on the CAM-UV algorithm, which infers causal variables for each observed variable. TS-CAM-UV uses time priority as prior knowledge for CAM-UV-PK, indicating that variables occurring later in time cannot be the cause of earlier variables. To the best of our knowledge, this is the first method for time series causal discovery that adopts a causal function model approach assuming the presence of latent confounders. If the data being analyzed satisfy the assumption that the causal function takes the form of a generalized additive model, then this proposed method can accurately infers causal relationships even in the presence of latent confounders.

Future research will extend our approach to models where the causal graph contains cycles. If the time for the causal effect from the cause variable to the effect variable is shorter than the time slice of the data being analyzed, this causal effect becomes a contemporaneous effect. When there is a causal relationship such as $X_{i}^{t-2}\rightarrow X_{j}^{t-1}\rightarrow X_{i}^{t}$, and the time slice of the data is longer than this causal effect, it results in a contemporaneous effect with cycles. Therefore, future research explore causal discovery methods that allow for cycles in contemporaneous effects.


%
%



%
%

\bibliographystyle{acm}
\bibliography{ref.bib}

\end{document}